\documentclass[reqno]{amsart}
\usepackage{graphicx}
\usepackage{hyperref}

\theoremstyle{plain}
\newtheorem{theorem}{Theorem}[section]

\theoremstyle{remark}
\newtheorem{remark}{Remark}[section]

\numberwithin{equation}{section}

\title[A feedforward network with only one neuron in the hidden layer]{A single hidden layer feedforward network with only one neuron in the hidden layer can approximate any univariate function}

\author{Namig J. Guliyev}
\address{Institute of Mathematics and Mechanics, Azerbaijan National Academy of Sciences, 9 B.~Vahabzadeh str., AZ1141, Baku, Azerbaijan.}
\email{njguliyev@gmail.com}

\author{Vugar E. Ismailov}
\address{Institute of Mathematics and Mechanics, Azerbaijan National Academy of Sciences, 9 B.~Vahabzadeh str., AZ1141, Baku, Azerbaijan.}
\email{vugaris@mail.ru}



\begin{document}
\maketitle
\begin{abstract}
The possibility of approximating a continuous function on a compact subset of the real line by a feedforward single hidden layer neural network with a sigmoidal activation function has been studied in many papers. Such networks can approximate an arbitrary continuous function provided that an unlimited number of neurons in a hidden layer is permitted. In this paper, we consider constructive approximation on any finite interval of $\mathbb{R}$ by neural networks with only one neuron in the hidden layer. We construct algorithmically a smooth, sigmoidal, almost monotone activation function $\sigma$ providing approximation to an arbitrary continuous function within any degree of accuracy. This algorithm is implemented in a computer program, which computes the value of $\sigma$ at any reasonable point of the real axis.
\end{abstract}

\section{Introduction}

Neural networks are being successfully applied across an extraordinary range of problem domains, in fields as diverse as computer science, finance, medicine, engineering, physics, etc. The main reason for such popularity is their ability to approximate arbitrary functions. For the last 30 years a number of results have been published showing that the artificial neural network called a \emph{feedforward network with one hidden layer} can approximate arbitrarily well any continuous function of several real variables. These results play an important role in determining boundaries of efficacy of the considered networks. But the proofs are usually do not state how many neurons should be used in the hidden layer. The purpose of this paper is to prove constructively that a neural network having only one neuron in its single hidden layer can approximate arbitrarily well all continuous functions defined on any compact subset of the real axis.

The building blocks for neural networks are called \emph{neurons}. An artificial neuron is a device with $n$ real inputs and an output. This output is generally a superposition of a univariate function with an affine function in the $n$-dimensional Euclidean space, that is a function of the form $\sigma(w_{1} x_{1} + \cdots + w_{n} x_{n} - \theta)$. The neurons are organized in layers. Each neuron of each layer is connected to each neuron of the subsequent (and thus previous) layer. Information flows from one layer to the subsequent layer (thus the term feedforward). A feedforward neural network with one hidden layer has three layers: input layer, hidden layer, and output layer. A feedforward network with one hidden layer consisting of $r$ neurons computes functions of the form
\begin{equation} \label{eq:intro}
c_{0} + \sum_{i=1}^{r} c_{i} \sigma (w_{i} \cdot x - \theta_{i}).
\end{equation}
Here the vectors $w_{i}$, called \emph{weights}, are vectors in $\mathbb{R}^{n}$; the \emph{thresholds} $\theta_{i}$ and the coefficients $c_{i}$ are real numbers and $\sigma$ is a univariate activation function. The following are common examples of activation functions:
\begin{align*}
\sigma(t) &= \frac{1}{1+e^{-t}} &&\text{(the squashing function),}\\
\sigma(t) &= \begin{cases}
0, &t \leq -1, \\
\dfrac{t+1}{2}, &-1 \leq t \leq 1, \\
1, &t \geq 1
\end{cases} &&\text{(the piecewise linear function),}\\
\sigma(t) &= \frac{1}{\pi }\arctan t+\frac{1}{2} &&\text{(the arctan sigmoid function),}\\
\sigma(t) &= \frac{1}{\sqrt{2\pi }}\int\limits_{-\infty }^{t}e^{-x^{2}/2}dx &&\text{(the Gaussian function).}
\end{align*}

In many applications, it is convenient to take the activation function $\sigma$ as a \emph{sigmoidal function} which is defined as
\begin{equation*}
\lim_{t \to -\infty} \sigma(t) = 0 \quad \text{ and } \quad \lim_{t \to +\infty} \sigma(t)=1.
\end{equation*}
The literature on neural networks abounds with the use of such functions and their superpositions. Note that all the above activation functions are sigmoidal.

In approximation by neural networks, there are two main problems. The first is the \emph{density} problem of determining the conditions under which an arbitrary target function can be approximated arbitrarily well by neural networks. The second problem, called the \emph{complexity} problem, is to determine how many neurons in hidden layers are necessary to give a prescribed degree of approximation. This problem is almost the same as the problem of degree of approximation (see \cite{Barron, Cao, Hahm}). The possibility of approximating a continuous function on a compact subset of the real line (or $n$-dimensional space) by a single hidden layer neural network with a sigmoidal activation function has been well studied in a number of papers. Different methods were used. Carroll and Dickinson \cite{Carroll} used the inverse Radon transformation to prove the universal approximation property of single hidden layer neural networks. Gallant and White \cite{Gallant} constructed a specific continuous, nondecreasing sigmoidal function, called a cosine squasher, from which it was possible to obtain any Fourier series. Thus their activation function had the density property. Cybenko \cite{Cybenko} and Funahashi, \cite{Funahashi} independently from each other, established that feedforward neural networks with a continuous sigmoidal activation function can approximate any continuous function within any degree of accuracy on compact subsets of $\mathbb{R}^{n}.$ Cybenko's proof uses the functional analysis method, combining the Hahn--Banach Theorem and the Riesz Representation Theorem, whiles Funahashi's proof applies the result of Irie and Miyake \cite{Irie} on the integral representation of $f \in L_{1}(\mathbb{R}^{n})$ functions, using a kernel which can be expressed as a difference of two sigmoidal functions. Hornik, Stinchcombe and White \cite{HornikSt} applied the Stone--Weierstrass theorem, using trigonometric functions.

K\r{u}rkov\'{a} \cite{Kurkova2} proved that staircase like functions of any sigmoidal type can approximate continuous functions on any compact subset of the real line within an arbitrary accuracy. This is effectively used in K\r{u}rkov\'{a}'s subsequent results, which show that a continuous multivariate function can be approximated arbitrarily well by two hidden layer neural networks with a sigmoidal activation function (see~\cite{Kurkova1, Kurkova2}).

Chen, Chen and Liu \cite{Chen} extended the result of Cybenko by proving that any continuous function on a compact subset of $\mathbb{R}^{n}$ can be approximated by a single hidden layer feedforward network with a bounded (not necessarily continuous) sigmoidal activation function. Almost the same result was independently obtained by Jones \cite{Jones}.

Costarelli and Spigler \cite{Costarelli} reconsidered Cybenko's approximation theorem and for a given function $f \in C[a,b]$ constructed certain sums of the form (\ref{eq:intro}), which approximate $f$ within any degree of accuracy. In their result, similar to \cite{Chen}, $\sigma$ is bounded and sigmoidal. Therefore, when $\sigma \in C(\mathbb{R})$, the result can be viewed as a density result in $C[a,b]$ for the set of all functions of the form (\ref{eq:intro}).

Chui and Li \cite{Chui} proved that a single hidden layer network with a continuous sigmoidal activation function having integer weights and thresholds can approximate an arbitrary continuous function on a compact subset of $\mathbb{R}$. Ito \cite{Ito} established a density result for continuous functions on a compact subset of $\mathbb{R}$ by neural networks with a sigmoidal function having only unit weights. Density properties of a single hidden layer network with a restricted set of weights were studied also in other papers (for a detailed discussion see \cite{Ismailov1}).

In many subsequent papers, which dealt with the density problem, nonsigmoidal activation functions were allowed. Among them are the papers by Stinchcombe and White \cite{Stinchcombe}, Cotter \cite{Cotter}, Hornik \cite{Hornik}, Mhaskar and Micchelli \cite{Mhaskar}, and other researchers. The more general result in this direction belongs to Leshno, Lin, Pinkus and Schocken \cite{Leshno}. They proved that the necessary and sufficient condition for any continuous activation function to have the density property is that it not be a polynomial. For a detailed discussion of most of the results in this section, see the review paper by Pinkus \cite{Pinkus}.

It should be remarked that in all the above mentioned works the number of neurons $r$ in the hidden layer is not fixed. As such to achieve a desired precision one may take an excessive number of neurons. This, in turn, gives rise to the problem of complexity (see above).

Our approach to the problem of approximation by single hidden layer feedforward networks is different and quite simple. We consider networks (\ref{eq:intro}) defined on $\mathbb{R}$ with a limited number of neurons ($r$ is fixed!) in a hidden layer and ask the following fair question: is it possible to construct a well behaved (that is, sigmoidal, smooth, monotone, etc.) universal activation function providing approximation to arbitrary continuous functions on any compact set in $\mathbb{R}$ within any degree of precision? We show that this is possible even in the case of a feedforward network with only one neuron in its hidden layer (that is, in the case $r = 1$). The basic form of our theorem claims that there exists a smooth, sigmoidal, almost monotone activation function $\sigma$ with the property: for each univariate continuous function $f$ on the unit interval and any positive $\varepsilon$ one can chose three numbers $c_{0}$, $c_{1}$ and $\theta$ such that the function $c_{0} + c_{1}\sigma(t - \theta)$ gives $\varepsilon$-approximation to $f$. It should be remarked that we prove not only the existence result but also give an algorithm for constructing the mentioned universal sigmoidal function. For a wide class of Lipschitz continuous functions we also give an algorithm for evaluating the numbers $c_{0}$, $c_{1}$ and $\theta$.

For numerical experiments we used SageMath \cite{Sage}. We wrote a code for creating the graph of $\sigma$ and computing $\sigma(t)$ at any reasonable $t \in \mathbb{R}$. The code is open-source and available at {\tt http://sites.google.com/site/njguliyev/papers/sigmoidal}.

\section{The theoretical result}

We begin this section with the definition of a $\lambda$-increasing ($\lambda$-decreasing) function. Let $\lambda$ be any nonnegative number. A real function $f$ defined on $(a, b)$ is called $\lambda$-increasing ($\lambda$-decreasing) if there exists an increasing (decreasing) function $u \colon (a, b) \to \mathbb{R}$ such that $\left\vert f(x) - u(x)\right\vert
\le \lambda$, for all $x \in (a, b)$. If $u$ is strictly increasing (or strictly decreasing), then the above function $f$ is called a $\lambda$-strictly increasing (or $\lambda$-strictly decreasing) function. Clearly, $0$-monotonicity coincides with the usual concept of monotonicity and a $\lambda_{1}$-increasing function is $\lambda_{2}$-increasing if $\lambda_{1} \le \lambda_{2}$.

The following theorem is valid.

\begin{theorem}
For any positive numbers $\alpha$ and $\lambda$, there exists a $C^{\infty}(\mathbb{R})$, sigmoidal activation function $\sigma \colon \mathbb{R} \to \mathbb{R}$ which is strictly increasing on $(-\infty, \alpha)$, $\lambda$-strictly increasing on $[\alpha, +\infty)$, and satisfies the following property: For any finite closed interval $[a, b]$ of $\mathbb{R}$ and any $f \in C[a, b]$ and $\varepsilon > 0$ there exist three real numbers $c_{0}$, $c_{1}$ and $\theta$ for which
\begin{equation*}
\left| f(t) - c_{1} \sigma \left( \frac{\alpha}{b-a} t - \theta \right) - c_{0} \right| < \varepsilon
\end{equation*}
for all $t \in [a, b]$.
\end{theorem}
\begin{proof}
Let $\alpha$ be any positive number. Divide the interval $[\alpha, +\infty)$ into the segments $[\alpha, 2\alpha]$, $[2\alpha, 3\alpha]$, $\ldots$. Let $h(t)$ be any strictly increasing, infinitely differentiable function on $[\alpha, +\infty)$ with the properties
\begin{enumerate}
  \item $0 < h(t) < 1$ for all $t \in [\alpha, +\infty)$;
  \item $1 - h(\alpha) \le \lambda$;
  \item $h(t) \to 1$, as $t \to +\infty$.
\end{enumerate}
The existence of a strictly increasing smooth function satisfying these properties is easy to verify. Note that from conditions (1)--(3) it follows that any function $w(t)$ satisfying the inequality $h(t) < w(t) < 1$ for all $t \in [\alpha, +\infty)$, is $\lambda$-strictly increasing and $w(t) \to 1$, as $t \to +\infty$.

We are going to construct $\sigma$ obeying the required properties in stages. Let $\{u_{n}(t)\}_{n=1}^{\infty}$ be the sequence of all polynomials with rational coefficients defined on $[0,1].$ First, we define $\sigma$ on the closed intervals $[(2m-1)\alpha, 2m\alpha ]$, $m = 1, 2, \ldots$, as the function
\begin{equation*}
\sigma(t) = a_m + b_m u_m \left( \frac{t}{\alpha} - 2m + 1 \right), \quad t \in [(2m-1)\alpha, 2m\alpha],
\end{equation*}
or equivalently,
\begin{equation} \label{eq:sigma}
  \sigma(\alpha t + (2m-1)\alpha) = a_m + b_m u_m(t), \quad t \in [0, 1],
\end{equation}
where $a_{m}$ and $b_{m} \neq 0$ are chosen in such a way that the condition
\begin{equation} \label{eq:h_sigma_1}
  h(t) < \sigma(t) < 1
\end{equation}
holds for all $t \in [(2m-1)\alpha, 2m\alpha]$.

At the second stage we define $\sigma$ on the intervals $[2m\alpha, (2m+1)\alpha]$, $m = 1, 2, \ldots$, so that it is in $C^{\infty}(\mathbb{R})$ and satisfies the inequality (\ref{eq:h_sigma_1}). Finally, in all of $(-\infty, \alpha)$ we define $\sigma$ while maintaining the $C^{\infty}$ strict monotonicity property, and also in such a way that $\lim_{t \to -\infty} \sigma(t) = 0$. We obtain from the properties of $h$ and the condition (\ref{eq:h_sigma_1}) that $\sigma(t)$ is a $\lambda$-strictly increasing function on the interval $[\alpha, +\infty)$ and $\sigma(t) \to 1$, as $t \to +\infty$. Note that the construction of a $\sigma$ obeying all the above conditions is feasible. We show this in the next section.

From (\ref{eq:sigma}) it follows that for each $m = 1$, $2$, $\ldots$,
\begin{equation} \label{eq:u_m}
  u_m(t) = \frac{1}{b_m} \sigma(\alpha t + (2m-1)\alpha) - \frac{a_m}{b_m}.
\end{equation}

Let now $g$ be any continuous function on the unit interval $[0, 1]$. By the density of polynomials with rational coefficients in the space of continuous functions over any compact subset of $\mathbb{R}$, for any $\varepsilon > 0$ there exists a polynomial $u_{m}(t)$ of the above form such that
\begin{equation*}
\left\vert g(t)-u_{m}(t)\right\vert < \varepsilon,
\end{equation*}
for all $t \in [0, 1]$. This together with (\ref{eq:u_m}) means that
\begin{equation} \label{eq:g_epsilon}
  \left\vert g(t) - c_1 \sigma (\alpha t - s) - c_0 \right\vert < \varepsilon,
\end{equation}
for some $c_{0}$, $c_{1}$, $s \in \mathbb{R}$ and all $t \in [0, 1]$.

Note that (\ref{eq:g_epsilon}) proves our theorem for the unit interval $[0,1]$. Using linear transformation it is not difficult to go from $[0,1]$ to any finite closed interval $[a,b]$. Indeed, let $f \in C[a,b]$, $\sigma$ be constructed as above and $\varepsilon$ be an arbitrarily small positive number. The transformed function $g(t) = f(a + (b-a)t)$ is well defined on $[0,1]$ and we can apply the inequality (\ref{eq:g_epsilon}). Now using the inverse transformation $t=\frac{x-a}{b-a}$, we can write that
\begin{equation} \label{eq:f_epsilon}
  \left\vert f(x) - c_1 \sigma (wx - \theta ) - c_0 \right\vert < \varepsilon,
\end{equation}
where $w = \frac{\alpha}{b-a}$ and $\theta = \frac{\alpha a}{b-a}+s$. The last inequality (\ref{eq:f_epsilon}) completes the proof.
\end{proof}
\begin{remark}
The idea of using a limited number of neurons in hidden layers of a feedforward network was first implemented by Maiorov and Pinkus \cite{Maiorov}. They proved the existence of a sigmoidal, strictly increasing, analytic activation function such that two hidden layer neural networks with this activation function and a fixed number of neurons in each hidden layer can approximate any continuous multivariate function over the unit cube in $\mathbb{R}^{n}$. Note that the result is of theoretical value and the authors do not suggest constructing and using their sigmoidal function. Using the techniques developed in \cite{Maiorov}, we showed theoretically that if we replace the demand of analyticity by smoothness and monotonicity by $\lambda$-monotonicity, then the number of neurons in hidden layers can be reduced substantially (see \cite{Ismailov2}). We stress again that in both papers the algorithmic implementation of the obtained results is not discussed nor illustrated by numerical examples.
\end{remark}

In the next section, we propose an algorithm for computing the above sigmoidal function $\sigma$ at any point of the real axis. The code of this algorithm is available at {\tt http://sites.google.com/site/njguliyev/papers/sigmoidal}. As examples, we include in the paper the graph of $\sigma$ (see Figure 1) and a numerical table (see Table 1) containing several computed values of this function.

\section{Algorithmic construction of the universal sigmoidal function}

\textbf{Step 1.} \textit{Definition of $h(t)$.}

Set
$$h(t) := 1 - \frac{\min\{1/2, \lambda\}}{1 + \log(t - \alpha + 1)}.$$
Note that this function satisfies the conditions 1)--3) in the proof of Theorem 2.1.

\textbf{Step 2.} \textit{Enumerating the rationals.}

Let $a_n$ be Stern's diatomic sequence:
$$a_1 = 1, \quad a_{2n} = a_n,\ a_{2n+1} = a_n + a_{n+1},\ n = 1, 2, \dots.$$
It should be remarked that this sequence first appeared in print in 1858 \cite{Stern} and has been the subject of many papers (see, e.g., \cite{Northshield} and the references therein).

The Calkin--Wilf \cite{Calkin} sequence $q_n := a_n / a_{n+1}$ contains every positive rational number exactly once and hence the sequence
$$r_0 := 0, \quad r_{2n} := q_n, \quad r_{2n-1} := -q_n, \ n = 1, 2, \dots,$$
is the enumeration of all the rational numbers. It is possible to calculate $q_n$ and $r_n$ directly. Let
$$(\underbrace{1 1 \ldots 1}_{f_k}\underbrace{0 0 \ldots 0}_{f_{k-1}}\underbrace{1 1 \ldots 1}_{f_{k-2}} \ldots \underbrace{0 0 \ldots 0}_{f_1}\underbrace{1 1 \ldots 1}_{f_0})_2$$
be the binary code of $n$. Here, $f_0, f_1, \ldots, f_k$ show the number of 1-digits, 0-digits, 1-digits, etc., respectively, starting from the end of the binary code. Note that $f_0$ can be zero. Then $q_n$ equals the continued fraction
\begin{equation} \label{eq:fraction}
[f_0; f_1, \ldots, f_k] := f_0 + \dfrac1{f_1 + \dfrac1{f_2 + \dfrac1{\ddots + \dfrac1{f_k}}}}.
\end{equation}
The calculation of $r_n$ is reduced to the calculation of $q_{n/2}$, if $n$ is even and $q_{(n+1)/2}$, if $n$ is odd.

\textbf{Step 3.} \textit{Enumerating the polynomials with rational coefficients.}

It is clear that every positive rational number determines a unique finite continued fraction $[n_0; n_1, \ldots, n_l]$ with $n_0 \ge 0$, $n_1, \ldots, n_{l-1} \ge 1$ and $n_l \ge 2$.

Since each non-zero polynomial with rational coefficients can uniquely be written as $r_{k_0} + r_{k_1} t + \ldots + r_{k_d} t^d$, where $r_{k_d} \ne 0$ (i.e. $k_d > 0$), we have the following bijection between the set of all non-zero polynomials with rational coefficients and the set of all positive rational numbers:
$$r_{k_0} + r_{k_1} t + \ldots + r_{k_d} t^d \mapsto [k_0; k_1 + 1, \ldots, k_d + 1]$$
We define $u_1(t) := 0$ and
$$u_m(t) := r_{n_0} + r_{n_1-1} t + \ldots + r_{n_d-1} t^d, \quad m = 2, 3, \ldots,$$
where $q_{m-1} = [n_0; n_1, \ldots, n_d]$.

\textbf{Step 4.} \textit{Construction of $\sigma$ on $[(2m-1)\alpha, 2m\alpha]$.}

Set $M := h((2m+1)\alpha)$. Besides, for each polynomial $u_m(t) = d_0 + d_1 t + \ldots + d_k t^k$, set
\begin{equation} \label{eq:A_1}
  A_1 := d_0 + \frac{d_1-|d_1|}{2} + \ldots + \frac{d_k-|d_k|}{2}
\end{equation}
and
\begin{equation} \label{eq:A_2}
  A_2 := d_0 + \frac{d_1+|d_1|}{2} + \ldots + \frac{d_k+|d_k|}{2}.
\end{equation}
It is not difficult to verify that
$$A_1 \le u_m(t) \le A_2, \quad t \in [0,1].$$

If $u_m$ is constant, then we put
$$\sigma(t) := \frac{1+M}{2}.$$
Otherwise, we define $\sigma (t)$ as the function
\begin{equation} \label{eq:sigma_again}
\sigma(t) = a_m + b_m u_m \left( \frac{t}{\alpha} - 2m + 1 \right), \quad t \in [(2m-1)\alpha, 2m\alpha],
\end{equation}
where
\begin{equation} \label{eq:a_m, b_m}
  a_m := \frac{(1+2M) A_2 - (2+M) A_1}{3(A_2 - A_1)}, \qquad b_m := \frac{1-M}{3(A_2 - A_1)}.
\end{equation}
Note that $a_m$, $b_m$ are the coefficients of the linear function $y=a_m+b_mx$ mapping the closed interval $[A_1, A_2]$ onto the closed interval $[(1+2M)/3, (2+M)/3]$. Thus,
\begin{equation} \label{eq:h_M_sigma_1}
  h(t) < M < \frac{1+2M}{3} \le \sigma(t) \le \frac{2+M}{3} < 1,
\end{equation}
for all $t \in [(2m-1)\alpha, 2m\alpha]$.

\begin{figure}
  \includegraphics[width=0.75\textwidth]{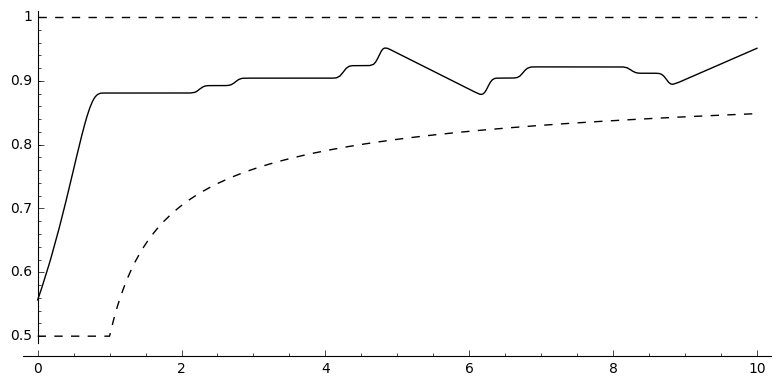}
  \caption{The graph of $\sigma$ on $[0, 10]$ ($\alpha = 1$, $\lambda = 1/2$)}
\end{figure}

\begin{table}
  \begin{tabular}{|c|c|c|c|c|c|c|c|c|c|} \hline
$t$ & $\sigma$ & $t$ & $\sigma$ & $t$ & $\sigma$ & $t$ & $\sigma$ & $t$ & $\sigma$ \\ \hline
$\alpha$ & $0.88087$ & $100$ & $0.95548$ & $200$ & $0.96034$ & $0$ & $0.55682$ & $-100$ & $0.00868$ \\ \hline
$10$ & $0.95095$ & $110$ & $0.94162$ & $210$ & $0.96064$ & $-10$ & $0.07655$ & $-110$ & $0.00790$ \\ \hline
$20$ & $0.95879$ & $120$ & $0.97124$ & $220$ & $0.94790$ & $-20$ & $0.04096$ & $-120$ & $0.00725$ \\ \hline
$30$ & $0.96241$ & $130$ & $0.97163$ & $230$ & $0.96119$ & $-30$ & $0.02796$ & $-130$ & $0.00670$ \\ \hline
$40$ & $0.96464$ & $140$ & $0.94397$ & $240$ & $0.97430$ & $-40$ & $0.02122$ & $-140$ & $0.00623$ \\ \hline
$50$ & $0.94931$ & $150$ & $0.97230$ & $250$ & $0.95743$ & $-50$ & $0.01710$ & $-150$ & $0.00581$ \\ \hline
$60$ & $0.96739$ & $160$ & $0.97259$ & $260$ & $0.97461$ & $-60$ & $0.01432$ & $-160$ & $0.00545$ \\ \hline
$70$ & $0.93666$ & $170$ & $0.94573$ & $270$ & $0.95793$ & $-70$ & $0.01232$ & $-170$ & $0.00514$ \\ \hline
$80$ & $0.96910$ & $180$ & $0.94622$ & $280$ & $0.94979$ & $-80$ & $0.01081$ & $-180$ & $0.00485$ \\ \hline
$90$ & $0.93951$ & $190$ & $0.94669$ & $290$ & $0.96670$ & $-90$ & $0.00963$ & $-190$ & $0.00460$ \\ \hline
  \end{tabular}
  \caption{Some computed values of $\sigma$ ($\alpha = 1$, $\lambda = 1/2$)}
\end{table}

\textbf{Step 5.} \textit{Construction of $\sigma$ on $[2m\alpha, (2m+1)\alpha]$.}

To define $\sigma$ on the intervals $[2m\alpha, (2m+1)\alpha]$ we will use the \emph{smooth transition function}
$$\beta_{a,b}(t) := \frac{\widehat{\beta}(b-t)}{\widehat{\beta}(b-t) + \widehat{\beta}(t-a)},$$
where
$$\widehat{\beta}(t) := \begin{cases} e^{-1/t}, & t > 0, \\ 0, & t \le 0. \end{cases}$$
It is easy to see that $\beta_{a,b}(t) = 1$ for $t \le a$, $\beta_{a,b}(t) = 0$ for $t \ge b$ and $0 < \beta_{a,b}(t) < 1$ for $a < t < b$. Set
$$K := \frac{\sigma(2m\alpha) + \sigma((2m+1)\alpha)}{2}$$
Since both $\sigma(2m\alpha)$ and $\sigma((2m+1)\alpha)$ belong to the interval $(M, 1)$, we obtain that $K \in (M, 1)$.

First we extend $\sigma$ smoothly to the interval $[2m\alpha, 2m\alpha + \alpha/2]$. Take $\varepsilon := (1 - M) / 6$ and choose $\delta \le \alpha/2$ such that
\begin{equation} \label{eq:epsilon}
  \left| a_m + b_m u_m \left( \frac{t}{\alpha} - 2m + 1 \right) - \left( a_m + b_m u_m(1) \right) \right| \le \varepsilon, \quad t \in [2m\alpha, 2m\alpha + \delta].
\end{equation}
Let us show how one can choose this $\delta$. If $u_m$ is constant, it is sufficient to take $\delta := \alpha/2$. If $u_m$ is not constant, we take
$$\delta := \min\left\{ \frac{\varepsilon\alpha}{b_m C}, \frac{\alpha}{2} \right\},$$
where $C$ is a number satisfying $|u'_m(t)| \le C$ for $t \in (1, 1.5)$. Now define $\sigma$ on the first half of the interval $[2m\alpha, (2m+1)\alpha]$ as the function
\begin{equation} \label{eq:sigma_left}
\begin{split}
  \sigma(t) & := K - \beta_{2m\alpha, 2m\alpha + \delta}(t) \\
  & \times \left(K - a_m - b_m u_m \left( \frac{t}{\alpha} - 2m + 1 \right)\right), \quad t \in \left[ 2m\alpha, 2m\alpha + \frac{\alpha}{2} \right].
\end{split}
\end{equation}

Let us verify that $\sigma(t)$ satisfies the condition (\ref{eq:h_sigma_1}). Indeed, if $2m\alpha + \delta \le t \le 2m\alpha + \frac{\alpha}{2}$, then there is nothing to prove, since $\sigma(t) = K \in (M, 1)$. If $2m\alpha \le t < 2m\alpha + \delta$, then $0 < \beta_{2m\alpha, 2m\alpha + \delta}(t) \le 1$ and hence from (\ref{eq:sigma_left}) it follows that for each $t \in [2m\alpha, 2m\alpha + \delta)$, $\sigma(t)$ is between the numbers $K$ and $A(t) = a_m + b_m u_m \left( \frac{t}{\alpha} - 2m + 1 \right)$. On the other hand, from (\ref{eq:epsilon}) we obtain that
\begin{equation*}
  a_m + b_m u_m(1) - \varepsilon \le A(t) \le a_m + b_m u_m(1) + \varepsilon,
\end{equation*}
which together with (\ref{eq:sigma_again}) and (\ref{eq:h_M_sigma_1}) yields that $A(t) \in \left[ \frac{1 + 2M}{3} - \varepsilon, \frac{2 + M}{3} + \varepsilon \right]$, for $t \in [2m\alpha, 2m\alpha + \delta)$. Since $\varepsilon = (1 - M) / 6$, the inclusion $A(t) \in (M, 1)$ is valid. Now since both $K$ and $A(t)$ belong to $(M, 1)$, we finally conclude that
\begin{equation*}
  h(t) < M < \sigma(t) < 1, \quad \text{for } t \in \left[ 2m\alpha, 2m\alpha + \frac{\alpha}{2} \right].
\end{equation*}

We define $\sigma$ on the second half of the interval in a similar way:
\begin{equation*}
\begin{split}
  \sigma(t) & := K - (1 - \beta_{(2m+1)\alpha - \overline{\delta}, (2m+1)\alpha}(t)) \\
  & \times \left(K - a_{m+1} - b_{m+1} u_{m+1} \left( \frac{t}{\alpha} - 2m - 1 \right)\right), \quad t \in \left[ 2m\alpha + \frac{\alpha}{2}, (2m+1)\alpha \right],
\end{split}
\end{equation*}
where
$$\overline{\delta} := \min\left\{ \frac{\overline{\varepsilon}\alpha}{b_{m+1} \overline{C}}, \frac{\alpha}{2} \right\}, \qquad \overline{\varepsilon} := \frac{1 - h((2m+3)\alpha)}{6}, \qquad \overline{C} \ge \sup_{[-0.5, 0]} |u'_{m+1}(t)|.$$
One can easily verify, as above, that the constructed $\sigma(t)$ satisfies the condition (\ref{eq:h_sigma_1}) on $[2m\alpha, 2m\alpha + \alpha/2]$ and
$$\sigma \left( 2m\alpha + \frac{\alpha}{2} \right) = K, \qquad \sigma^{(i)} \left( 2m\alpha + \frac{\alpha}{2} \right) = 0, \quad i = 1, 2, \ldots.$$

\textbf{Step 6.} \textit{Construction of $\sigma$ on $(-\infty, \alpha)$.}

Finally, we put
$$\sigma(t) := \left( 1 - \widehat{\beta}(\alpha - t) \right) \frac{1 + h(3\alpha)}{2}, \quad t \in (-\infty, \alpha).$$
It is not difficult to verify that $\sigma$ is a strictly increasing, smooth function on $(-\infty, \alpha)$. Note also that $\sigma(t) \to \sigma(\alpha) = (1 + h(3\alpha))/2$ (see Step 4), as $t$ tends to $\alpha$ from the left and $\sigma^{(i)}(\alpha) = 0$, for $i = 1, 2, \ldots$.

Step 6 completes the construction of the universal activation function $\sigma$, which satisfies Theorem 2.1.

\section{The algorithm for evaluating the numbers $c_0$, $c_1$ and $\theta$}

Although Theorem 2.1 is valid for all continuous functions, in practice it is quite difficult to calculate algorithmically $c_0$, $c_1$ and $\theta$ in Theorem 2.1 for badly behaved continuous functions. The main difficulty arises while attempting to design an efficient algorithm for the construction of a best approximating polynomial within any given degree of accuracy. But for certain large classes of well behaved functions, the computation of the above numbers is doable. In this section, we show this for the class of Lipschitz continuous functions.

Assume that $f$ is a Lipschitz continuous function on $[a, b]$ with a Lipschitz constant $L$. In order to find the parameters $c_0$, $c_1$ and $\theta$ algorithmically, it is sufficient to perform the following steps.

\textbf{Step 1.} \textit{Going to the unit interval.}

Consider the function $g(t)=f(a+(b-a)t)$, which is Lipschitz continuous on $[0,1]$ with a Lipschitz constant $L_1=L(b-a)$. Denote by
$$B_n(x) := \sum_{k=0}^n g \left( \frac{k}{n} \right) \binom{n}{k} x^k (1-x)^{n-k}$$
the $n$-th Bernstein polynomial of the function $g$. Let $\varepsilon > 0$ be given.

\textbf{Step 2.} \textit{Finding the position of a given rational number.} 

Define the functions
$$p_q \colon \mathbb{Q}_{+} \to \mathbb{N}, \qquad p_r \colon \mathbb{Q} \to \mathbb{N} \cup \{0\},$$
which return the positions of a positive rational number and a rational number in the sequences $\{q_n\}$ and $\{r_n\}$, respectively (see Section 3). We start with the computation of $p_q$. Let $q$ be a positive rational number. If (\ref{eq:fraction}) is the continued fraction representation of $q$ with $k$ even (we may always consider $[f_0; f_1, \ldots, f_k - 1, 1]$ instead of $[f_0; f_1, \ldots, f_k]$ if needed), then the binary representation of the position $p_q(q)$ of $q$ in the Calkin--Wilf sequence is
$$p_q(q) = (\underbrace{1 1 \ldots 1}_{f_k}\underbrace{0 0 \ldots 0}_{f_{k-1}}\underbrace{1 1 \ldots 1}_{f_{k-2}} \ldots \underbrace{0 0 \ldots 0}_{f_1}\underbrace{1 1 \ldots 1}_{f_0})_2.$$
Now we can easily find $p_r(r)$ by the formula
$$p_r(r) = \begin{cases} 0, & r = 0, \\ 2p_q(r), & r > 0, \\ 2p_q(-r) - 1, & r < 0. \end{cases}$$

\textbf{Step 3.} \textit{Finding $n \in \mathbb{N}$ such that $|B_n(x) - g(x)| \le \varepsilon / 2$, $x \in [0, 1]$.}

We use the inequality (see \cite{Sikkema})
$$|B_n(x) - g(x)| \le \chi \omega\left(\frac{1}{\sqrt{n}}\right), \quad x \in [0, 1],$$
where
$$\chi = \frac{4306 + 837\sqrt{6}}{5832} = 1.089887\ldots.$$
Since
$$\omega\left(\frac{1}{\sqrt{n}}\right) \le \frac{L_1}{\sqrt{n}}$$
it is sufficient to take
$$n := \left\lceil \left( \frac{2\chi L_1}{\varepsilon} \right)^2 \right\rceil,$$
where $\lceil \cdot \rceil$ is the ceiling function defined as $\lceil x \rceil := \min \{ k \in \mathbb{Z} \mid k \ge x \}$.

\textbf{Step 4.} \textit{Finding a polynomial $p$ with rational coefficients such that $|p(x) - B_n(x)| \le \varepsilon / 2$, $x \in [0, 1]$.}

If $B_n(x) = a_0 + a_1 x + \ldots + a_k x^k$ then it is sufficient to choose $d_0$, $d_1$, \ldots, $d_k \in \mathbb{Q}$ such that
$$|a_0 - d_0| + |a_1 - d_1| + \ldots + |a_k - d_k| \le \frac{\varepsilon}{2}$$
and put $p(x) := d_0 + d_1 x + \ldots + d_k x^k$.

\textbf{Step 5.} \textit{Finding $m \in \mathbb{N}$ such that $u_m \equiv p$.}

For the definition of $u_m$ see Section 3. If $p \equiv 0$ then clearly $m = 1$. Otherwise let $n_i := p_r(d_i)$ be the positions of the numbers $d_i$ in the sequence $\{r_n\}$. Then
$$m = p_q([n_0; n_1 + 1, \ldots, n_k + 1]) + 1.$$

\textbf{Step 6.} \textit{Evaluating the numbers $c_0$, $c_1$ and $s$ in (\ref{eq:g_epsilon}).}

Set $s := (1 - 2m)\alpha$. If $u_m = p$ is constant then we put
$$c_1 := 1, \qquad c_0 := d_0 - \frac{1 + h((2m + 1)\alpha)}{2}.$$
If $u_m = p$ is not constant then we put
$$c_1 := \frac{1}{b_m}, \qquad c_0 := -\frac{a_m}{b_m},$$
where $a_m$ and $b_m$ are computed using the formulas (\ref{eq:A_1}), (\ref{eq:A_2}) and (\ref{eq:a_m, b_m}).

\textbf{Step 7.} \textit{Evaluating the numbers $c_0$, $c_1$ and $\theta$.}

In this step, we return to our original function $f$ and calculate the numbers $c_0$, $c_1$ and $\theta$ (see Theorem 2.1). The numbers $c_0$ and $c_1$ have been calculated above (they are the same for both $g$ and $f$). In order to find $\theta$, we use the formula $\theta =\frac{\alpha a}{b-a}+s$.

\begin{remark}
Note that some computational difficulties may arise while implementing the above algorithm in standard computers. For some functions, the index $m$ of a polynomial $u_m$ in Step 5 may be extraordinarily large. In this case, a computer is not capable of producing this number, hence the numbers $c_0$, $c_1$ and $\theta$.
\end{remark}

\end{document}